\documentclass[11pt]{article}

\usepackage[final]{acl}

\usepackage{times}
\usepackage{latexsym}

\usepackage[T1]{fontenc}

\usepackage[utf8]{inputenc}

\usepackage{microtype}
\usepackage{inconsolata}
\usepackage{graphicx}
\usepackage{booktabs}
\usepackage{tabularx}
\usepackage{threeparttable}
\usepackage{array}
\usepackage{makecell}
\usepackage{amsmath}
\usepackage{amssymb}
\usepackage{amsthm}
\usepackage{dsfont}
\usepackage[ruled,lined,noend]{algorithm2e}

\newtheorem{theorem}{Theorem}
\newtheorem{proposition}[theorem]{Proposition}
\newtheorem{assumption}[theorem]{Assumption}

\theoremstyle{remark}

\newcommand{\skill}[1]{\textsc{#1}}
\providecommand{\skill}[1]{\textsc{#1}}
\newcommand{\methodname}{\textsc{ReuseRL}}

\title{Skill Reuse as Compression in Agentic RL}

\author{ \textbf{Zhikun Xu\textsuperscript{1}},
 \textbf{Yu Feng\textsuperscript{2}},
 \textbf{Jacob Dineen\textsuperscript{1}},
 \textbf{Taiwei Shi\textsuperscript{3}},
 \textbf{Jieyu Zhao\textsuperscript{3}},
 \textbf{Ben Zhou\textsuperscript{1}}
\\[1ex]
 \textsuperscript{1}Arizona State University,
 \textsuperscript{2}University of Pennsylvania,\\
 \textsuperscript{3}University of Southern California
\\
\texttt{zhikunxu@asu.edu}}

\begin{document}
\maketitle

\begin{abstract}
Large language model agents trained with reinforcement learning (RL) often learn brittle, task-specific shortcuts. We hypothesize that agents generalize better when their successful trajectories are structurally compressible, decomposed into a small set of reusable abstract patterns. To formalize this, we introduce \methodname{}, which grounds agentic RL in the Minimum Description Length (MDL) principle. \methodname{} extracts a shared skill dictionary from successful trajectories and augments the RL objective with a segmentation cost, explicitly penalizing idiosyncratic behaviors that encode poorly. We prove a PAC-Bayes bound guaranteeing that a dictionary extracted from successful trajectories has bounded expected description length on future successful behavior. Across ALFWorld, TextWorld-Cooking, and Countdown-Stepwise, \methodname{} improves in- and out-of-distribution success over vanilla GRPO and strong round-length baselines.\footnote{\url{https://github.com/ARC-ASU/ReuseRL}}
\end{abstract}

\section{Introduction}\label{sec:intro}

Large language model (LLM) agents have demonstrated remarkable capabilities across complex interactive tasks, from embodied household reasoning \citep{shridhar2021alfworld} and text-based games \citep{cote2018textworld} to symbolic reasoning \citep{stojanovski2026reasoning}, by combining natural-language reasoning with environmental interaction \citep{ReAct,shinn2023reflexion}. Reinforcement learning (RL) has become a central training paradigm for these agents \citep{schulman2017ppo,shao2024deepseekmath}, with multi-turn agentic variants of GRPO now driving much of the recent progress on long-horizon decision-making \citep{jin2025searchr1,feng2025gigpo,xia2026skillrl,shi2026experiential}.

Yet a persistent challenge is that RL training often entrenches brittle behavioral shortcuts. This failure mode, sometimes termed reasoning collapse, arises when prolonged optimization amplifies memorized action patterns at the expense of transferable reasoning \citep{wang2025ragen,wang2026reasoning,wang2026ragen}. We do not claim this failure has a single cause; rather, we isolate one structural factor we can both formalize and act on: rewarding task success alone permits idiosyncratic, one-off solution patterns that satisfy the training tasks but transfer poorly. Crucially, compact, repeated solution patterns are desirable for generalization, so the difficulty is not that the policy forms shortcuts but that it also forms idiosyncratic ones that complete training tasks without composing into reusable skills. Recent efforts combat this by augmenting RL with reusable structure or experience via teacher distillation \citep{xia2026skillrl}, reflection loops \citep{shi2026experiential}, or external memory \citep{shinn2023reflexion,zhao2024expel,wang2023voyager}. But it remains unclear which shared structural property of the learned behavior is responsible for their gains.

\begin{figure*}[t]
\centering
\includegraphics[width=\textwidth]{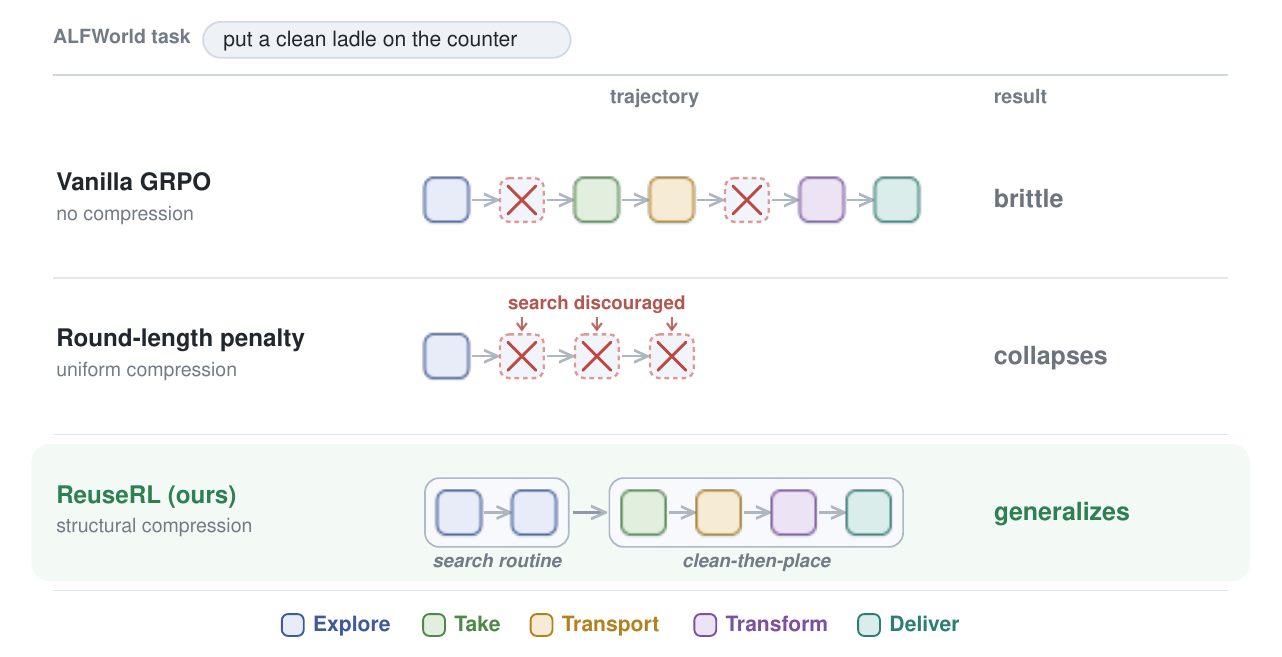}
\caption{\small \textbf{\methodname{} distinguishes reusable compression from raw brevity.} Each colored box is one ALFWorld atomic skill. Vanilla GRPO optimizes task success alone and can produce long trajectories with repeated or wasted steps. A pure round-length penalty is a degenerate singleton-only code that penalizes all steps uniformly, including the necessary search, so the agent under-explores and never reaches the target. \methodname{} learns a multi-skill dictionary from successful trajectories and uses the segmentation cost under this dictionary as a trajectory-level penalty, keeping reusable subroutines cheap while leaving idiosyncratic waste expensive.}
\label{fig:teaser}
\vspace{-1em}
\end{figure*}
We conjecture that many existing methods owe part of their gains to implicitly encouraging reusable structure. In both reinforcement learning and information theory, such reusable structure is closely tied to data compression and the Minimum Description Length (MDL) principle \citep{rissanen1978modeling,grunwald2007minimum}. This connection suggests a broader conjecture: successful, generalizable behaviors are precisely those that admit compact, reusable decompositions. We formalize this property as the structural compressibility of successful trajectories, defined as the degree to which successful behavior, represented as a sequence of abstract skills, is built by composing a small set of reusable patterns. As illustrated in Fig.~\ref{fig:teaser}, an agent whose successes decompose into a compact repertoire of reusable skill compositions reasons more robustly than one that relies on idiosyncratic, one-off sequences. This perspective naturally bridges temporal abstraction in hierarchical RL, where options compress behavior into reusable subroutines \citep{sutton1999between}, with the MDL principle. While BPE-based action tokenization \citep{sennrich2016neural,pmlr-v235-zheng24b} and MDL-guided skill discovery \citep{zhang2021minimum, jiang2022learning} have begun bridging these ideas by treating behavioral abstraction as compression, none integrate this compression directly into the RL training signal as a trajectory-level penalty for LLM agents. Consequently, the compression-generalization hypothesis remains untested in this setting.

To test this hypothesis, 
our contributions are as follows. 
\textbf{(1)} We first formalize the hypothesis that structural compressibility is a key driver of reasoning generalization in RL-trained LLM agents. \textbf{(2)} We introduce \methodname{}, a minimal framework that extracts a shared skill dictionary online from successful trajectories via greedy BPE-style merging and augments the per-trajectory GRPO signal with the resulting segmentation cost. \textbf{(3)} We ground our framework in the MDL principle via an EM-style alternating optimization, show that the pure round-length penalty is a degenerate fixed-code segmentation cost (Proposition \ref{prop:length}), establish a PAC-Bayes bound on the expected description length of successful behavior (Theorem \ref{thm:main}), and provide a unifying interpretation of SkillRL and ERL as implicit, partial MDL minimizers. \textbf{(4)} On ALFWorld \citep{shridhar2021alfworld}, TextWorld-Cooking \citep{cote2018textworld}, and Countdown-Stepwise \citep{stojanovski2026reasoning}, \methodname{} improves over both vanilla GRPO and the pure round-length baseline. Crucially, these results support our central premise: successful trajectories should not merely be short; they should be compressible under a reusable skill dictionary. We isolate this effect in a self-contained RL loop, deliberately omitting experience-augmented methods like SkillRL or ERL that rely on auxiliary supervision. This allows us to rigorously isolate and evaluate the effect of the MDL principle without introducing confounding factors. We emphasize that our objective is not to demonstrate superiority over experience-augmented systems or to establish a new state of the art.

\section{Related Work}\label{sec:related}
\subsection{Grammar-Based and MDL-Guided Hierarchical Abstraction}
Skill discovery can be framed as sequence compression, where hierarchical abstraction yields reusable subroutines. Grounded in the MDL principle~\citep{rissanen1978modeling,grunwald2007minimum}, prior methods learn compact skill codebooks that balance dictionary size against how well the data is described~\citep{zhang2021minimum,jiang2022learning}. In RL, BPE-style merging of frequent action patterns extends this idea to temporally extended macro-actions~\citep{kozma2024theoretical}. PRISE~\citep{pmlr-v235-zheng24b} applies compression-based abstractions to improve sample efficiency, and SkillVLA~\citep{zhai2026skillvla} leverages reusable skills for combinatorial task diversity in embodied settings. PAC results for state abstraction~\citep{cipollone2025realizable} further show that compressed representations can preserve near-optimal behavior. \methodname{} builds on this line of work but applies the MDL objective online, as a per-trajectory RL signal on successful behavior, instead of as an offline criterion for constructing reusable skills.

\subsection{Skill Consolidation and Experience-Augmented Reasoning}
LLM agents increasingly build skills through environmental interaction. During training, ERL \citep{shi2026experiential} and SkillRL \citep{xia2026skillrl} distill environmental feedback into behavioral corrections. Similarly, systems like RAGEN \citep{wang2025ragen} and DYSTIL \citep{wang2025dystil} self-improve through dense feedback, while others rely on retrieved past experiences \citep{shinn2023reflexion, zhao2024expel, wang2023voyager}. Left unchecked, however, accumulating self-generated trajectories can trigger reasoning collapse, amplifying memorized shortcuts at the expense of transferable behavior \citep{wang2025ragen,wang2026ragen}. This motivates explicit compression of reusable behavioral structure as a regularizer, which is exactly what \methodname{} adds without relying on teacher distillation, retrieval, or other auxiliary modules.

\subsection{Information-Theoretic Policy Compression}
Recent works apply information-theoretic objectives to constrain LLM reasoning traces, such as adaptively penalizing raw reasoning length \citep{wu2025lapo,yuan2026shortenyourerightlazy,lin2026boostingllmreasoninghumaninspired} or measuring token cost via surprisal under a language-model prior \citep{massoli2026reasoning}. Crucially, these approaches shape reasoning at the token or length level. In contrast, \methodname{} operates at the behavioral level by compressing the structure of successful skill sequences rather than their surface-level length.

\section{\methodname{}: Skill Reuse as Compression RL}\label{sec:method}
\subsection{Problem Setup and Skill Projection}\label{sec:setup}

We consider a task distribution $\mathcal{D}$ over multi-turn reasoning tasks. For each sampled environment $d \sim \mathcal{D}$, a language model policy $\pi_\theta$ interacts with a partially observable environment for at most $T$ steps, producing a trajectory $\tau \sim \pi_\theta(\cdot \mid d)$ of the form $\tau = (o_1, a_1, o_2, a_2, \ldots, o_T, a_T)$, where $o_t$ is the observation and $a_t$ is the action at step $t$. Each trajectory receives a binary terminal reward $R(\tau) \in \{0, 1\}$ indicating success or failure. The standard RL objective maximizes the expected per-trajectory success rate:
\begin{equation}\label{eq:rl-objective}
\max_\theta \; \mathbb{E}_{d \sim \mathcal{D},\; \tau \sim \pi_\theta(\cdot \mid d)} \bigl[ R(\tau) \bigr]
\end{equation}
In GRPO-style training~\citep{shao2024deepseekmath}, Eq.~\eqref{eq:rl-objective} is optimized through sampled batches $\mathcal{B} = \{(d_i,\tau_i)\}_{i=1}^{n}$ and per-trajectory advantages computed inside each batch. 
As argued in \S\ref{sec:intro}, the structural factor we isolate is the absence of pressure toward structurally reusable successful behavior under a success-only reward. The remainder of this section makes this notion precise: we represent behavior as a sequence of atomic skills, measure its reuse through a shared skill dictionary, and derive the per-trajectory signal that rewards compressible successful behavior.

\paragraph{Skill alphabet and projection.}
We represent behavior at a higher level as a sequence of discrete \emph{atomic skills}, where each skill corresponds to an interpretable primitive such as \textsc{Explore}, \textsc{Take}, etc. Let $\Sigma = \{\sigma_1, \ldots, \sigma_{K_\Sigma}\}$ denote a finite \emph{atomic skill alphabet}, where each $\sigma_i$ is a skill. A \emph{skill projection} $\varphi$ maps each trajectory $\tau$ to a skill sequence $s \in \Sigma^*$, abstracting away low-level action details: $\varphi: \tau \mapsto s \in \Sigma^*$. $K_\Sigma := |\Sigma|$ is the size of the skill alphabet and $^*$ denotes the Kleene star. In our instantiation, $|\Sigma|$ is small and environment-specific, and $\varphi$ is implemented by a rule-based mapping over the verbs of the raw actions; per-environment mapping rules are summarized in \S\ref{sec:experiments}.

\paragraph{Skill dictionary.}
To measure structural reuse, we represent successful behavior using a \emph{skill dictionary} $\mathcal{C} = \{p_1, \ldots, p_M\}$, where each entry $p_j \in \Sigma^{\leq L}$ is a short skill sub-sequence of length at most $L$. Each $p_j$ is therefore a reusable subsequence of atomic skills, and $|p_j|$ refers to the number of atomic skills in $p_j$. Given the group of successful trajectories $\mathcal{B}^+$, 
we restrict our attention to the universe of valid dictionaries, $\mathcal{A} := \{\mathcal{C} \subseteq \Sigma^{\le L} : \Sigma \subseteq \mathcal{C}\}$, which guarantees that every dictionary contains all atomic singleton skills. This ensures that any atomic skill sequence $s=\varphi(\tau), \tau\in\mathcal{B}^+$ can always be exactly represented using phrases from $\mathcal{C}$, in the worst case as a concatenation of singletons. 
Given $s$ and $\mathcal{C}$, let $\mathrm{seg}(s,\mathcal{C})$ denote the minimum number of dictionary phrases needed to cover $s$ exactly; we compute this quantity by dynamic programming. Smaller values of $\mathrm{seg}(s,\mathcal{C})$ indicate that $s$ is more compressible under $\mathcal{C}$. Since $\mathcal{B}^+ \subseteq \mathcal{C}^*$ 
and $\varphi$ assigns at most one skill label per environment step, every non-empty skill sequence satisfies: $1 \leq \mathrm{seg}(s,\mathcal{C}) \leq |s| \leq T$.

\subsection{Two-Part Description Length of the Successful Batch}\label{sec:bdl}

For a sampled batch $\mathcal{B} = \{(d_i,\tau_i)\}_{i=1}^{n}$, let $\mathcal{B}^+ := \{\tau_i : R(\tau_i)=1, \tau_i\in\mathcal{B}\}$ denote the subset of successful trajectories, and let $s_i = \varphi(\tau_i)$ be the skill sequence of $\tau_i \in \mathcal{B}^+$. When $\mathcal{B}^+ \neq \emptyset$, we evaluate a candidate skill dictionary $\mathcal{C}$ using a standard two-part code following the MDL principle~\citep{rissanen1978modeling,grunwald2007minimum}: one part encodes the shared dictionary itself, and the other encodes the successful skill sequences using phrases from that dictionary. We thus define the dictionary cost as
\begin{equation}\label{eq:dict-cost}
L_{\mathrm{skill\_dict}}(\mathcal{C}) = \sum_{j=1}^{M}\bigl(|p_j|\log_2 K_\Sigma + \log_2 L\bigr)
\end{equation}
which measures the cost of storing the $M$ phrases in $\mathcal{C}$: $|p_j|\log_2 K_\Sigma$ accounts for the cost of specifying the atomic skills in $p_j$, and $\log_2 L$ accounts for its length. The corresponding batch-level description length of $\mathcal{B}^+$ under $\mathcal{C}$ is
\begin{equation}\label{eq:bdl}
\begin{aligned}
\mathcal{L}_{\mathrm{BDL}}(\mathcal{B}^+, \mathcal{C})
&= \frac{L_{\mathrm{skill\_dict}}(\mathcal{C})}
        {\lvert \mathcal{B}^+ \rvert} \\
&+ \frac{1}{\lvert \mathcal{B}^+ \rvert}
\sum_{\tau_i \in \mathcal{B}^+}
\operatorname{seg}\!\left(\varphi(\tau_i), \mathcal{C}\right)
\log_2 \lvert \mathcal{C} \rvert
\end{aligned}
\end{equation}
The first term is the shared cost of the dictionary, amortized across successful trajectories. The second term is the average cost of reconstructing their skill sequences using phrases from $\mathcal{C}$. From a standard coding perspective, the dynamic programming segmentation partitions a sequence into $\operatorname{seg}(\varphi(\tau_i), \mathcal{C})$ distinct phrases. Assuming a uniform code over the dictionary entries, specifying each phrase requires $\log_2 |\mathcal{C}|$ bits to index. Therefore, their product yields the total bits required to encode the data given the dictionary. We select the batch dictionary by minimizing this objective:
\begin{equation}\label{eq:opt-dict}
\hat{\mathcal{C}}(\mathcal{B}^+) = \arg\min_{\mathcal{C}\in\mathcal{A}} \mathcal{L}_{\mathrm{BDL}}(\mathcal{B}^+,\mathcal{C})
\end{equation}
Equation~\eqref{eq:bdl} is the objective we \emph{minimize over $\mathcal{C}$} to extract the shared dictionary. The next subsection derives the corresponding \emph{per-trajectory} signal that the policy should optimize, by examining the structure of joint MDL minimization over $(\mathcal{C},\pi_\theta)$.

\subsection{EM-style Optimization of the MDL Objective: The segmentation cost}\label{sec:idio}

Jointly minimizing the description length over both the shared dictionary $\mathcal{C}$ and the policy $\pi_\theta$ naturally yields an EM-style alternating procedure: the current policy generates a batch of successful trajectories $\mathcal{B}^+$, from which we infer a compressive dictionary $\mathcal{C}$ (E-step); holding this dictionary fixed, we then update $\pi_\theta$ to favor successful trajectories that are cheaper to encode under $\mathcal{C}$ (M-step).

\paragraph{E-step (dictionary extraction).}
With $\pi_\theta$ fixed, we re-estimate the shared dictionary from the current successful batch $\mathcal{B}^+$ by solving Eq.~\eqref{eq:opt-dict}. The dictionary-storage term $L_{\mathrm{skill\_dict}}(\mathcal{C})/|\mathcal{B}^+|$ is constructive here: it controls \emph{which} shared structure we extract by trading dictionary size against per-trajectory segmentation cost.

\paragraph{M-step (policy update).}
With $\mathcal{C}$ fixed at $\mathcal{\mathcal{C}}(\mathcal{B}^+)$, we move $\pi_\theta$ toward successful trajectories whose skill sequences are cheap to encode under $\mathcal{C}$. Once $\mathcal{C}$ is fixed, the only part of the joint description length that varies with the individual trajectory is the conditional segmentation cost, and this is the quantity we attach to the per-trajectory RL signal.

Concretely, we use the \emph{segmentation cost} of a successful trajectory $\tau_i$ with skill sequence $s_i = \varphi(\tau_i)$ under dictionary $\mathcal{C}$ as.
\begin{equation}\label{eq:idio}
\mathrm{SegCost}(s_i \mid \mathcal{C}) := \frac{\mathrm{seg}(s_i,\mathcal{C})}{T}
\end{equation}
A trajectory that decomposes into a few long reusable phrases has $\mathrm{seg}(s,\mathcal{C}) \ll T$ and a lower segmentation cost; a trajectory composed of unrelated singleton skills has $\mathrm{seg}(s,\mathcal{C}) = |s| \approx T$ and the maximum segmentation cost of $1$. 

The augmented trajectory-level RL objective is therefore to maximize
\begin{equation}\label{eq:rl-mdl}
\begin{aligned}
&\mathop{\mathbb{E}}\limits_{\mathcal{B}\sim(\mathcal{D},\pi_\theta)^n}
\Biggl[
\frac{1}{n}\sum_{i=1}^{n}
\Bigl(
R(\tau_i) \\
&\quad -
\lambda\cdot \mathds{1}\{R(\tau_i)=1\}\,
\operatorname{SegCost}\bigl(
s_i \mid \mathcal{C}(\mathcal{B}^+)
\bigr)
\Bigr)
\Biggr]
\end{aligned}
\end{equation}
with $\lambda$ as the efficiency coefficient. The indicator restricts the penalty to successful trajectories, so the signal never discourages task completion. The full alternating procedure, i.e., E-step over $\mathcal{C}(\mathcal{B}^+)$, then M-step GRPO update over $\pi_\theta$, iteratively reduces the joint description length of the successful behavior the policy produces.


\subsection{Pure Round Length as Fixed-Code MDL}\label{sec:fixed-code}
Pure round-length penalties compress trajectories uniformly by discouraging all steps equally, regardless of whether those steps participate in reusable behavioral structure. In our framework, this corresponds to a degenerate fixed-code setting where the dictionary contains only singleton skills and is not allowed to merge into reusable phrases.

\begin{proposition}[Pure round length is the fixed-code segmentation cost]\label{prop:length}
Let $\mathcal{C}_\Sigma := \Sigma \in \mathcal{A}$ be the singleton-only dictionary. Then for every non-empty skill sequence $s$,
\[\mathrm{seg}(s, \mathcal{C}_\Sigma) = |s| \quad \text{and} \quad \mathrm{SegCost}(s \mid \mathcal{C}_\Sigma) = \frac{|s|}{T}.\]
For the singleton-only dictionary, segmentation cost is exactly pure round length. For any dictionary containing a useful non-singleton phrase, there exist equal-length sequences that receive different segmentation costs. Thus, raw length cannot, in general, distinguish reusable from non-reusable structure. (Proof in \S\ref{app:proofs}.)
\end{proposition}
The proposition has a direct empirical consequence. Uniform brevity cannot distinguish sequences with the same raw length but different segmentation costs and fail to favor reusable compositional structure across different successful trajectories.

\subsection{Generalization Bound on the MDL Objective}\label{sec:theory}

We now theoretically analyze the MDL objective utilized in the EM-style optimization of \S\ref{sec:idio} by deriving a PAC-Bayes bound on the expected description length of successful skill sequences. Importantly, this theorem bounds the expected description length under a fixed policy distribution, rather than characterizing policy generalization after iterative RL updates. The argument guarantees that if an extracted dictionary tightly compresses an observed successful batch, its expected description length on future successful trajectories from the same distribution is formally bounded. While this does not alone prove that the per-trajectory penalty acts as a generalization-improving regularizer, it formally justifies using the empirical MDL objective as a reliable measure for evaluating the structural reusability of successful behaviors.

For a given $\pi_\theta$, let $Q_\theta^+$ denote the distribution over successful skill sequences induced by $d \sim \mathcal{D}$, $\tau \sim \pi_\theta(\cdot \mid d)$, $s = \varphi(\tau)$, and $R(\tau) = 1$.

\begin{assumption}[Compositional skill structure]\label{asm:compositional}
The successful skill-sequence distribution $Q_\theta^+$ admits a shared reusable decomposition: there exists a dictionary $\mathcal{C}^\star \in \mathcal{A}$ such that (i)~the dictionary itself has bounded complexity, $L_{\mathrm{skill\_dict}}(\mathcal{C}^\star) \leq B^\star$; and (ii)~every successful skill sequence $s$ in the support of $Q_\theta^+$ can be segmented using at most $K^\star$ phrases from $\mathcal{C}^\star$, i.e., $\mathrm{seg}(s,\mathcal{C}^\star) \leq K^\star$. Here $B^\star$ and $K^\star$ are distribution-dependent constants determined by the task family, skill alphabet, and phrase-length cap; they do not scale with the sample size.
\end{assumption}

This assumption states that successful reasoning trajectories are generated by composing a bounded number of reusable skill patterns from a shared vocabulary, a condition naturally satisfied in structured environments such as ALFWorld, where diverse household tasks decompose into common subroutines (e.g., \textsc{Explore} $\to$ \textsc{Take} $\to$ \textsc{Transform} $\to$ \textsc{Deliver}).

\begin{theorem}[Trajectory-level MDL generalization bound]\label{thm:main}
Let $P(\mathcal{C}) \propto 2^{-L_{\mathrm{skill\_dict}}(\mathcal{C})}$, $\mathcal{C} \in \mathcal{A}$, be a prior over admissible dictionaries fixed before observing any successful training sample. 
Draw $m \geq 1$ i.i.d.\ successful skill sequences $S_m = (s_1, \ldots, s_m) \sim (Q_\theta^+)^m$, and let $\widehat{\mathcal{C}} = \widehat{\mathcal{C}}(S_m)$ be any dictionary extracted from this successful sample. To normalize the per-trajectory MDL loss, define
\[
\begin{aligned}
U_m(\mathcal{C})
&= \frac{L_{\mathrm{skill\_dict}}(\mathcal{C})}{m}
   + T\log_2|\mathcal{C}|, \\[0.5em]
\widetilde{\ell}_m(s,\mathcal{C})
&= \frac{
      L_{\mathrm{skill\_dict}}(\mathcal{C})/m
      + \operatorname{seg}(s,\mathcal{C})\log_2|\mathcal{C}|
   }{
      U_m(\mathcal{C})
   }
\end{aligned}
\]
Then $\widetilde{\ell}_m(s,\mathcal{C}) \in [0,1]$ for every non-empty successful skill sequence $s$. Moreover, for any $\delta > 0$, with probability at least $1 - \delta$ over the draw of $S_m$,
\begin{equation}\label{eq:gen-bound}
\begin{aligned}
\mathbb{E}_{s \sim Q_\theta^+}
&\!\left[\widetilde{\ell}_m(s,\widehat{\mathcal{C}})\right]
\leq
\frac{1}{m}\sum_{i=1}^{m}
\widetilde{\ell}_m(s_i,\widehat{\mathcal{C}}) \\
&+
\sqrt{
\frac{
L_{\mathrm{skill\_dict}}(\widehat{\mathcal{C}})\ln 2
+\ln(m/\delta)
}{
2m
}
}
\end{aligned}
\end{equation}
The expected normalized MDL of the extracted dictionary on a fresh successful skill sequence drawn from $Q_\theta^+$ is therefore bounded by its empirical normalized MDL on the observed successful sample, plus a complexity term that decreases with $m$ and increases with the description length of $\widehat{\mathcal{C}}$. The proof, together with a corollary translating $B^\star$ and $K^\star$ into an explicit bound on future description length under $\mathcal{C}^\star$, is provided in \S\ref{app:proofs}.
\end{theorem}



\subsection{SkillRL and ERL as Implicit MDL Minimizers}\label{sec:implicit-mdl}

Recent skill-reuse frameworks implicitly compress different parts of the behavioral pipeline but do not explicitly optimize the shared successful-trajectory objective in Eq.~\eqref{eq:rl-mdl}. SkillRL~\citep{xia2026skillrl} constructs a hierarchical skill library via teacher-model distillation, effectively compressing the dictionary term $L_{\mathrm{skill\_dict}}$ but not penalizing per-trajectory segmentation complexity during RL updates. ERL~\citep{shi2026experiential} embeds an experience--reflection--consolidation loop that compresses instance-level correction traces into model weights but does not enforce shared compositional patterns across trajectories. In contrast, \methodname{} explicitly minimizes the full batch-level MDL objective by the EM process, simultaneously searching for a compact shared codebook and incentivizing reusable skill compositions among successful trajectories. Formal objective comparisons for these two methods are provided in \S\ref{app:implicit-mdl}.

\subsection{Algorithm and Implementation}\label{sec:algorithm}

We implement \methodname{} as a reward-shaping module within the GRPO training loop. Algorithm~\ref{alg:reuserl} in the appendix summarizes the procedure.

\paragraph{Atomic skill extraction.} Given a successful trajectory $\tau$, we extract its skill sequence $\varphi(\tau) \in \Sigma^*$ via a deterministic rule-based mapping over the raw action verbs. We intentionally choose a small, task-generic alphabet to capture reusable structure without collapsing into near-raw actions or losing critical behavioral distinctions. Mapping rules are summarized in \S\ref{sec:skill_extraction}.

\paragraph{BPE dictionary extraction with a global success buffer.}
We approximate the optimal dictionary $\mathcal{C}$ by a greedy Byte Pair Encoding (BPE) algorithm~\citep{10.5555/177910.177914}. Starting from singleton phrases (one per atomic skill in $\Sigma$), we repeatedly propose an adjacent phrase pair $(p_a, p_b)$ for merging, where candidate pairs are ranked by frequency in the successful cluster $\mathcal{T}$, which consists of the current successful batch $\mathcal{B}^+$ and an optional global successful buffer $\mathcal{G}$. The global buffer, collecting successful trajectories across batches in a FIFO structure, is to stabilize the E-step across small successful batches. A candidate merge is accepted only if adding the concatenated phrase $p_a \circ p_b$ strictly decreases the batch objective:
\[
\begin{aligned}
\Delta_{\mathrm{BDL}}
&=
\mathcal{L}_{\mathrm{BDL}}\!\left(
  \mathcal{T}_b,
  \mathcal{C}\cup\{p_a\circ p_b\}
\right) \\
&\quad
-
\mathcal{L}_{\mathrm{BDL}}\!\left(
  \mathcal{T}_b,\mathcal{C}
\right)
< 0
\end{aligned}
\]
We cap phrase length at $L = 4$ to avoid overfitting to specific trajectory-level details. This greedy procedure approximates the NP-hard optimal dictionary search~\citep{kozma2024theoretical} and terminates when no merge yields a negative MDL delta, which is further validated in \S\ref{sec:ablations}.


\section{Experimental Settings}\label{sec:experiments}
\subsection{Tasks}
We evaluate \methodname{} on three text-based agentic benchmarks: \textbf{ALFWorld}~\citep{shridhar2021alfworld} for household interaction tasks, \textbf{TextWorld-Cooking}~\citep{cote2018textworld} for recipe-following with irreversible processing errors, and \textbf{Countdown-Stepwise}~\citep{stojanovski2026reasoning} for arithmetic planning with rollback and reset actions. More details are in \S\ref{appendix:experiment_details}.

\subsection{Atomic Skill Extraction}\label{sec:skill_extraction}
Based on each task, we manually define deterministic rule-based skill projection over raw actions. ALFWorld uses 5 navigation/manipulation skills, TextWorld-Cooking uses 10 recipe-execution skills, and Countdown-Stepwise uses 26 target-relative arithmetic skills. These skill alphabets are empirically designed verb-effect equivalence classes that preserve behaviorally meaningful distinctions while abstracting away low-level action variations. A finer alphabet approaches the raw action space and reduces abstraction, whereas a coarser alphabet merges strategically distinct behaviors that are important for solving the task. More details are in \S\ref{app:skills}.
\subsection{Training and Evaluation}
All experiments are trained with GRPO for 300 steps. We use Qwen2.5-1.5B-Instruct \citep{qwen2.5} for ALFWorld and Countdown-Stepwise, and Qwen3-1.7B \citep{yang2025qwen3} in non-thinking mode for TextWorld-Cooking. For ALFWorld, we report SC@7~\citep{wang2023selfconsistency}: for each test scene, we sample $7$ rollouts and determine success by majority vote. It has two test splits: \textbf{IID} and \textbf{OOD}. For TextWorld-Cooking and Countdown-Stepwise, we report Pass@1 averaged over three independent runs, reporting the mean and standard deviation across seeds. More implementation and evaluation details are in \S\ref{appendix:experiment_details}.

\subsection{Compared Methods}
We compare four configurations per environment: (1)~\textbf{Vanilla}: the base model without any RL training; (2)~\textbf{Vanilla GRPO}: standard GRPO without any compression signal; (3)~\textbf{pure-round-length}: GRPO augmented with the penalty $\cdot|s|/T$ on successful trajectories, which Proposition~\ref{prop:length} shows is the fixed-code degenerate instance of \methodname{}; (4)~\textbf{\methodname{}-SegCost (Ours)}: GRPO augmented with the penalty $seg(s, \mathcal{C})/T$, where the dictionary $\mathcal{C}$ is extracted from the successful batch combined with the global buffer.

\section{Main Results}\label{sec:main}


\begin{table*}[t]
  \centering
  \small
  \begin{tabular*}{\textwidth}{@{\extracolsep{\fill}}lcccc@{}}
    \toprule
    & \multicolumn{2}{c}{\textbf{ALFWorld} (SC@7)} & \textbf{TW-Cooking} & \textbf{Countdown-Stepwise} \\
    \cmidrule(lr){2-3}
    \textbf{Method} & IID & OOD & Pass@1 & Pass@1 \\
    \midrule
    Vanilla            & 3.57  & 1.49  & 6.80\,$\pm$\,1.14  & 20.57\,$\pm$\,0.98 \\
    Vanilla GRPO       & 84.29 & 79.85 & 74.97\,$\pm$\,0.76 & 68.46\,$\pm$\,0.45 \\
    Pure Round-Length  & 96.43 & 91.79 & 64.03\,$\pm$\,0.25 & 77.02\,$\pm$\,0.96 \\
    \midrule
    \methodname{}-SegCost (no buffer) & 93.57 & 89.55 & \textbf{83.50\,$\pm$\,0.10} & 68.94\,$\pm$\,0.36 \\
    \methodname{}-SegCost             & \textbf{97.14} & \textbf{93.28} & 81.73\,$\pm$\,0.23 & \textbf{80.37\,$\pm$\,0.33} \\
    \bottomrule
  \end{tabular*}
  \caption{\small Main results. Bold marks the best score per column, attained by a \methodname{}-SegCost variant on every benchmark. Both variants are contributions of this work, and the unmarked \methodname{}-SegCost uses the global success buffer, whose role we analyze in \S\ref{sec:ablations}. ALFWorld reports SC@7 on IID and OOD splits of 140 and 134 scenes. TW-Cooking and Countdown-Stepwise report Pass@1 over 1000 and 1024 test instances as mean $\pm$ std across three seeds.}
  \label{tab:main}
  \vspace{-1em}
\end{table*}

\paragraph{Structural compression translates to improved reasoning generalization.} 
Crucially, the gains in Table~\ref{tab:main} demonstrate that \methodname{} improves beyond vanilla GRPO and the uniform compression. The strongest evidence is its out-of-distribution and compositional transfer: \methodname{} achieves its largest absolute improvement on the ALFWorld OOD split and held-out data from TW-Cooking and Countdown-Stepwise.

\paragraph{Compression robustly beats GRPO in environments dominated by step-budget exhaustion.}
On ALFWorld and Countdown-Stepwise, all compression methods outperform vanilla GRPO by large margins ($+11.94$ / $+8.56\% $ for pure-round-length on ALFWorld OOD / Countdown Pass@1, and another $+1.49$ / $+3.35\% $ for segmentation cost), because the intermediate states are highly reversible by actions. Invalid or suboptimal actions largely waste the step budget rather than triggering immediate fatal traps. This means any structurally consistent compression signal is informative.

\paragraph{Uniform length pressure backfires when redundant actions trigger irreversible failures.}
On TextWorld-Cooking, pure-round-length drops significantly, $10.94$ points even below vanilla GRPO. After the case study, failures concentrate on \emph{\texttt{cook=False} + cookbook-says-cook} recipes, whose \texttt{cook} gate is off because the ingredient arrives already in cooked form (e.g.\ \texttt{roasted red apple} on the counter) even though the cookbook still mentions the cooking of the ingredients, so attempting to cook the ingredient a second time burns it. More analysis about it is detailed in \S\ref{sec:case-cooking}.

\paragraph{A learned multi-step dictionary recovers task-conditional structure.}
\methodname{}-SegCost improves over pure-round-length on all three benchmarks ($+0.71$/$+1.49$\% ALFWorld IID/OOD, $+17.70$\% TW-Cooking, $+3.35$\% Countdown). The TW-Cooking gap is largest because pure-round-length minimizes raw step count and falls into the brittle cookbook-literal shortcut, whereas the segmentation cost rewards skill subsequences that recur across successful trajectories and so favors the recipe-faithful macro that generalizes. The strict empirical ordering \methodname{}-SegCost $\succ$ pure-round-length $\succ$ Vanilla GRPO holds on every benchmark except the middle term on TW-Cooking, where pure-round-length falls below GRPO. This inversion, predicted by Proposition~\ref{prop:length}, shows that effective compression must capture structure rather than raw length.

\section{Ablations}\label{sec:ablations}

\subsection{Effectiveness of Greedy BPE Dictionary Extraction}

The shared dictionary objective $\mathcal{L}_{\mathrm{BDL}}$ (Eq.~\ref{eq:bdl}) is NP-hard to minimize exactly over the admissible class $\mathcal{A}$, so the E-step in Algorithm~\ref{alg:reuserl} uses a greedy BPE merge schedule. On a synthetic corpus of dummy skill sequences, this schedule comes within $0.14\%$ of the matched brute-force optimum, recovers $99.02\%$ of its non-singleton phrases, and runs $285\times$ and $2231\times$ faster than the two brute-force variants. The greedy upper bound is thus tight enough that minimizing it inline effectively reduces $\mathcal{L}_{\mathrm{BDL}}$ over $\mathcal{A}$ at a cost low enough to run alongside every GRPO update. Full protocol and baselines are in \S\ref{app:bpe-merging}.

\subsection{Global Success Buffer}
Table~\ref{tab:main} reports an ablation without the global success buffer. The buffer is beneficial when successful trajectories within a single batch are too sparse or diverse for BPE to identify stable reusable phrases. Removing it lowers ALFWorld SC@7 from $97.14$/$93.28$\% to $93.57$/$89.55$\% on IID/OOD, and drops Countdown Pass@1 from $80.37$\% to $68.94$\%. However, the buffer can be harmful when dominated by easy trajectory patterns. On TextWorld-Cooking, removing the buffer raises TW-Cooking Pass@1 from $81.73$\% to $83.50$\%, entirely on the hard split ($9.0\%\!\to\!18.5\%$ solved, while simple games saturate at $99.8\%$ either way). A persistent buffer over-represents majority-class macros from simple games, pruning the interleaved cut-and-cook structures required for hard recipes, whereas a thinner per-batch corpus preserves this minority structure. Overall, the global success buffer serves as an optional E-step stabilizer: it improves phrase discovery when per-batch successes are sparse, but may hinder learning under highly imbalanced task distributions where majority patterns suppress minority structures. More details are in \S\ref{app:ablation_global_buffer}.

\subsection{Additional Ablations on Skill Projection and Model Scale}
To further validate our conclusion that the observed performance gains stem from MDL-based compression over a reasonable abstraction, we conduct additional ablation studies. Specifically, we evaluate a random skill projection baseline (\S\ref{app:skill_projection}) and extend our experiments to larger models (\S\ref{app:larger_models}).

\section{Case Study}\label{sec:case-study}

We now examine how environment structure shapes the behavior induced by different compression objectives. Table~\ref{tab:case-behavior} reports trajectory-level statistics at training step~300 that reveal failure modes and behavioral trade-offs beyond success rate alone.

\begin{table}[t]
\centering
\footnotesize
\setlength{\tabcolsep}{4.5pt}
\renewcommand{\arraystretch}{1.15}
\begin{tabularx}{\columnwidth}{@{}>{\raggedright\arraybackslash}X|cccc@{}}
\toprule
\textbf{Metric}
& \textbf{GRPO}
& \makecell[c]{\textbf{Pure}\\\textbf{Len.}}
& \makecell[c]{\textbf{SegCost}\\\textbf{(buf)}}
& \makecell[c]{\textbf{SegCost}\\\textbf{(no-buf)}} \\
\cmidrule(lr){2-5}
\multicolumn{5}{@{}l}{\textbf{ALFWorld (IID successes)}} \\
Success length, steps & 15.3 & 8.7 & 9.3 & \textbf{8.4} \\
Invalid, \%           & 23.5 & \textbf{1.4} & 1.8 & 3.9 \\
Repeat, \%            & 40.4 & \textbf{9.4} & 14.4 & 9.9 \\
\midrule
\multicolumn{5}{@{}l}{\textbf{TW-Cooking (all episodes)}} \\
Success length, steps & 12.1 & \textbf{8.2} & 9.0 & 8.8 \\
Burned failures, \%   & 21.1 & \textbf{74.0} & 3.3 & 9.7 \\
\midrule
\multicolumn{5}{@{}l}{\textbf{Countdown-Stepwise (all episodes)}} \\
0-\textsc{Reset} success,~\%   & 58.0 & 82.0 & 80.3 & \textbf{99.9} \\
\textsc{Mul}/\textsc{Div} OPs, \% & 3.1  & 6.8  & \textbf{10.9} & 5.5 \\
\bottomrule
\end{tabularx}
\caption{\small Combined non-SR behavioral signals at training step~300. ``Invalid'' (ALFWorld) is the fraction of successful actions returning ``Nothing happens.'' ``Repeat'' is the fraction of repeated actions within an episode. ``Burned'' marks TW-Cooking failures containing ``burned.'' Countdown-Stepwise failures are either timeouts (30-step limit) or stuck episodes caused by invalid or unparsable actions.
}

\vspace{-1em}
\label{tab:case-behavior}
\end{table}

\subsection{ALFWorld: Efficiency-Driven Gains over GRPO}\label{sec:case-alfworld}
ALFWorld's only failure mode is timeout, so step budget is the binding constraint; GRPO leaks it to invalid loops and action repetition, and both \methodname{} variants cut average successful-episode length from $15.3$ to $\sim\!9$ steps. The residual SegCost-vs-round-length gap localizes to complex multi-object tasks (\S\ref{app:case-study}); several high-frequency phrases recovered by \methodname{}-segmentation cost coincide with skills distilled in SkillRL's Skillbank~\citep{xia2026skillrl}, with the correspondence listed in Table~\ref{tab:skill-mapping}, \S\ref{app:case-study}.

\subsection{TextWorld-Cooking: The Burned-Food Failure of Pure Round Length}\label{sec:case-cooking}
Pure-round-length collapses onto the \texttt{Burned} failure mode ($74\%$ of its failures vs.\ $21\%$ for GRPO, Table~\ref{tab:case-behavior}) with the shortest failed trajectories of any method ($10.6$ steps), while \methodname{}-SegCost achieves comparable successful-episode efficiency with substantially fewer burn failures.
The mechanism---cookbook-literal \texttt{cook} on already-cooked ingredients---is policy-side for failing to learn the skill structures over different successful trajectories and triggering the biases of emitting the cook actions. We note that this failure mode was identified post hoc, through error
analysis of round-length training runs on TW-Cooking, rather than selected
in advance. More details about the failing are in \S\ref{app:ablation_global_buffer}, and Table~\ref{tab:tw-mechanism}, \S\ref{app:case-study}.

\subsection{Countdown-Stepwise: Solver-Template Formation and Operation Diversity}\label{sec:case-countdown}
Compression accelerates straight-through (zero-\textsc{Reset}) solving and uniquely expands \textsc{Mul}/\textsc{Div} usage under \methodname{}-SegCost, the signature of a learned dictionary that captures multi-step solver templates (e.g., executing a large subtraction to approach the target followed immediately by a precise division to clear the remainder). Per-method closing-OP counts and the full 
failure-mode partition are in \S\ref{app:ablation_global_buffer} and \S\ref{app:case-study}.

\section{Conclusion}\label{sec:conclusion}
We presented \methodname{}, augmenting GRPO with a segmentation cost $seg(s, \mathcal{C})/T$ derived from an EM-style MDL objective. By extracting a shared dictionary of reusable skills and penalizing idiosyncratic, one-off action sequences, \methodname{} regularizes against reasoning collapse such as wasting budgets and overfitting to certain curriculum in a task. This compression forces the agent to internalize transferable logic rather than memorize task-specific shortcuts, directly improving reasoning. Empirically, this generalization is evidenced by \methodname{}'s significant gains on the ALFWorld OOD split, its avoidance of deceptive traps in TextWorld-Cooking, and its systematic formation of multi-step solver templates in Countdown-Stepwise.

\section*{Limitations}

\methodname{} has four main limitations. \textbf{(i)}~The skill projection $\varphi$ is a rule-based per-environment mapping over action verbs; transferring \methodname{} to a new domain requires a small amount of manual schema design to abstract the atomic skills of completing the task (although the mapping itself is short and easily verifiable, and we consider this trade-off favorable to relying on an external LLM parser). \textbf{(ii)}~The greedy BPE dictionary search is an approximation to the NP-hard optimal-dictionary problem~\citep{kozma2024theoretical}; it is sufficient for the alphabets considered here, but the approximation gap is expected to grow with $|\Sigma|$ and may require alternative search heuristics in larger-vocabulary domains. \textbf{(iii)}~The generalization argument relies on Assumption~\ref{asm:compositional}: successful trajectories must admit a shared reusable decomposition with bounded dictionary complexity and bounded per-trajectory segment count. The assumption is a mild structural condition for the agentic benchmarks studied here, but can fail in unstructured or purely-novel reasoning domains where successful trajectories share less reusable substructure, in which case the segmentation cost signal degenerates toward the round-length penalty. \textbf{(iv)}~Our experiments are restricted to text-based agentic benchmarks and $1.5$--$1.7$B-parameter base models; extending \methodname{} to vision-language agents and to larger backbones is left to future work.

\section*{Use of LLMs}
We have NOT used LLMs to originate research ideas, NOT used LLMs to write original content in the paper, NOT used LLMs to generate data or plots.

\bibliography{custom}

\appendix
\begin{figure*}[t]
\centering
\includegraphics[width=\textwidth]{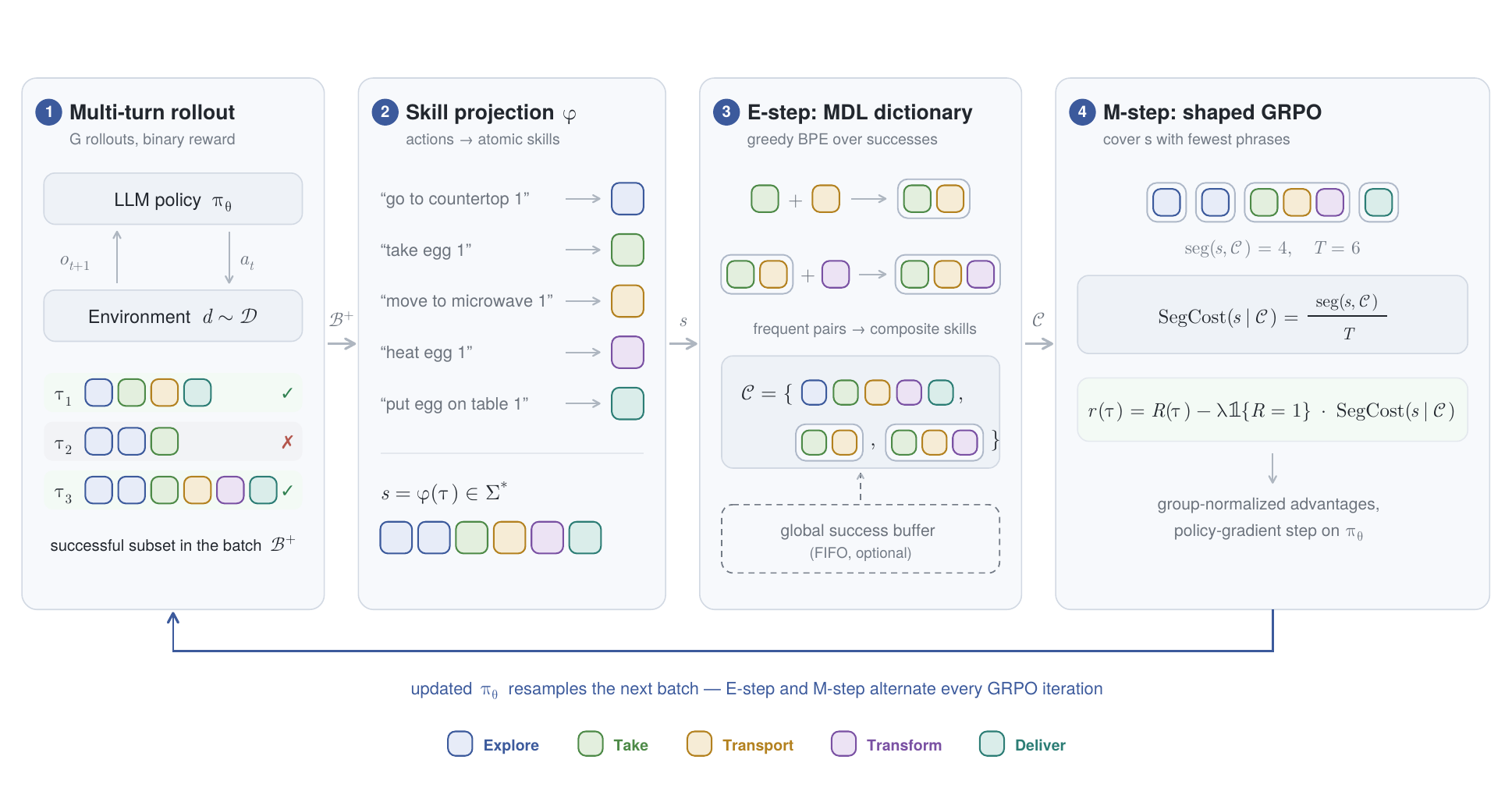}
\caption{\small \textbf{Overview of \methodname{}, illustrated in ALFWorld.} Colored boxes are the same five ALFWorld atomic skills as in Figure~\ref{fig:teaser}. \textbf{(1)}~The policy rolls out $G$ trajectories per task and receives only a binary terminal reward; the successful subset $\mathcal{B}^{+}$ is retained. \textbf{(2)}~A fixed rule-based projection $\varphi$ maps each environment action to an atomic skill, turning every trajectory in $\mathcal{B}^{+}$ into a skill string $s = \varphi(\tau) \in \Sigma^{*}$. \textbf{(3)}~In the E-step, greedy BPE merges frequent adjacent pairs into composite skills, accepting a merge only when it decreases the two-part description length, which yields the dictionary $\mathcal{C}$; an optional FIFO buffer of past successes stabilizes extraction when the current batch is small. \textbf{(4)}~In the M-step, $s$ is covered by the fewest phrases available in $\mathcal{C}$, and the resulting segmentation cost shapes the reward as $r(\tau) = R(\tau) - \lambda\,\mathds{1}\{R = 1\}\cdot \mathrm{SegCost}(s \mid \mathcal{C})$. Group-normalized advantages then drive a standard GRPO update, and the updated policy resamples the next batch, so dictionary extraction and policy optimization alternate at every iteration.}
\label{fig:overview}
\vspace{-1em}
\end{figure*}
\section{Overview Figure of \methodname{}}
Here is a high-level framework diagram illustrating the full \methodname{} pipeline in Fig.~\ref{fig:overview}.

\section{Proofs and Corollary}\label{app:proofs}
\subsection{Proof of Proposition~\ref{prop:length}}

\begin{proof}
The first claim is immediate: when $\mathcal{C}_\Sigma = \Sigma$ contains only singleton phrases, the unique exact segmentation of any non-empty $s\in\Sigma^*$ is the per-skill decomposition, so $\mathrm{seg}(s,\mathcal{C}_\Sigma)=|s|$ and $\mathrm{SegCost}(s\mid \mathcal{C}_\Sigma) = |s|/T$, which is exactly the pure round-length penalty.

For the second claim, we prove it by counterexample. Let $\Sigma=\{\sigma_1,\sigma_2,\sigma_3\}$, $\mathcal{C}=\{(\sigma_1),(\sigma_2),(\sigma_3),(\sigma_1,\sigma_2)\}$, $s_A=(\sigma_1,\sigma_2)^{T/2}$, and $s_B=(\sigma_1,\sigma_3,\sigma_2,\sigma_3)^{T/4}$ for $T$ divisible by $4$ gives $\mathrm{seg}(s_A,\mathcal{C})=T/2$ and $\mathrm{seg}(s_B,\mathcal{C})=T$, with the dictionary-storage term identical for both trajectories so the difference in segmentation cost comes entirely from the segmentation term.

Combining the two claims, the round-length penalty $|s|/T$ coincides with $\mathrm{SegCost}(s\mid\mathcal{C})$ only when $\mathcal{C}=\mathcal{C}_\Sigma$, and in every other admissible dictionary there exist trajectories with identical raw length but different segmentation costs; the round-length penalty is therefore not a consistent estimator of the segmentation cost under any non-degenerate dictionary.
\end{proof}

\subsection{Proof of Theorem~\ref{thm:main}}

\begin{proof}
Define the prior
\[
P(\mathcal{C}) = \frac{2^{-L_{\mathrm{skill\_dict}}(\mathcal{C})}}{Z},
\\
Z = \sum_{\mathcal{C}'\in\mathcal{A}} 2^{-L_{\mathrm{skill\_dict}}(\mathcal{C}')}
\]
fixed before observing any successful training sample. We first show that $Z\leq 1$. From Eq.~\eqref{eq:dict-cost}, the per-phrase code length is $c(p)=|p|\log_2 K_\Sigma + \log_2 L$, and $L_{\mathrm{skill\_dict}}(\mathcal{C})=\sum_{p\in\mathcal{C}} c(p)$, so
\[
Z = \Big(\frac{1}{LK_\Sigma}\Big)^{K_\Sigma}\prod_{\ell=2}^{L}\Big(1+\frac{1}{LK_\Sigma^{\ell}}\Big)^{K_\Sigma^{\ell}}
\]
Using $\ln(1+x)\leq x$,
\[
\ln Z \leq -K_\Sigma \ln(LK_\Sigma) + \frac{L-1}{L} \leq 0,
\]
hence $Z\leq 1$.

For a sample $S_m=(s_1,\ldots,s_m)\sim (Q_\theta^+)^m$ and a dictionary $\mathcal{C}\in\mathcal{A}$, define the normalized loss
\[
d_m(\mathcal C) := \frac{L_{\mathrm{skill\text{-}dict}}(\mathcal C)}{m}
\]
\[
U_m(\mathcal C) := d_m(\mathcal C) + T \log_2 |\mathcal C|.
\]
For \(s \in S_m\), define the normalized loss
\[
\widetilde{\ell}_m(s,\mathcal C)
:=
\frac{
d_m(\mathcal C)
+
\operatorname{seg}(s,\mathcal C)\log_2 |\mathcal C|
}{
U_m(\mathcal C)
}.
\]
Because every non-empty successful skill sequence satisfies $1\leq \mathrm{seg}(s,\mathcal{C})\leq T$, we have $\widetilde{\ell}_m(s,\mathcal{C})\in[0,1]$. Write
\[
\begin{aligned}
L_m(Q)
&:=
\mathbb{E}_{\mathcal C\sim Q}
\mathbb{E}_{s\sim Q_\theta^+}
\bigl[
\widetilde{\ell}_m(s,\mathcal C)
\bigr],
\\
\widehat L_m(Q;S_m)
&:=
\mathbb{E}_{\mathcal C\sim Q}
\left[
\frac{1}{m}
\sum_{i=1}^{m}
\widetilde{\ell}_m(s_i,\mathcal C)
\right].
\end{aligned}
\]
Let $Q$ be the posterior point mass on the extracted dictionary $\widehat{\mathcal{C}}=\widehat{\mathcal{C}}(S_m)$. The standard PAC-Bayes bound for $[0,1]$-valued losses~\citep{mcallester2003pac} gives that, with probability at least $1-\delta$ over the draw of $S_m$,
\[
\mathrm{kl}\!\bigl(\widehat{L}_m(Q)\,\big\|\,L(Q)\bigr)
\leq
\frac{\mathrm{KL}(Q\|P) + \ln(m/\delta)}{m},
\]
where $\mathrm{kl}(q\|p)$ denotes the binary KL divergence. Since $Q$ is a point mass on $\widehat{\mathcal{C}}$,
\[
\begin{aligned}
\operatorname{KL}(Q\Vert P)
&= -\ln P(\widehat{\mathcal C})
\\
&= L_{\mathrm{skill\text{-}dict}}(\widehat{\mathcal C})\,\ln 2
   + \ln Z
\\
&\le
L_{\mathrm{skill\text{-}dict}}(\widehat{\mathcal C})\,\ln 2,
\end{aligned}
\]
using $Z\leq 1$. Substituting and applying Pinsker's inequality $|p-q|\leq\sqrt{\mathrm{kl}(p\|q)/2}$ yields
\[
L(Q) \leq \widehat{L}_m(Q) + \sqrt{\frac{L_{\mathrm{skill\_dict}}(\widehat{\mathcal{C}})\ln 2 + \ln(m/\delta)}{2m}}.
\]
Because $Q$ is a point mass on $\widehat{\mathcal{C}}$, the expectations over $\mathcal{C}\sim Q$ collapse:
\[
\begin{aligned}
L(Q)
&= \mathbb{E}_{s\sim Q_\theta^+}
   \!\left[\widetilde{\ell}_m(s,\widehat{\mathcal{C}})\right],\\
\widehat{L}_m(Q)
&= \frac{1}{m}\sum_{i=1}^{m}
   \widetilde{\ell}_m(s_i,\widehat{\mathcal{C}}).
\end{aligned}
\]
which establishes Eq.~\eqref{eq:gen-bound}.
\end{proof}

\subsection{Corollary on the Universal Dictionary}\label{app:corollary}

\begin{proposition}[Bound under the universal dictionary]\label{prop:corollary}
Under Assumption~\ref{asm:compositional}, 
for every \(s\) in the support of \(Q_\theta^+\), the dictionary
\(\mathcal C^\star\) satisfies
\[
\begin{aligned}
&\frac{L_{\mathrm{skill\text{-}dict}}(\mathcal C^\star)}{m}
+
\operatorname{seg}(s,\mathcal C^\star)
\log_2 |\mathcal C^\star|
\\
&\qquad\le
U_m^\star
:=
\frac{B^\star}{m}
+
K^\star \log_2 |\mathcal C^\star|.
\end{aligned}
\]
Applying the same PAC-Bayes argument to the normalized loss yields
\[
\psi_m^\star(s)
:=
\frac{L_{\mathrm{skill}\text{-}\mathrm{dict}}(\mathcal C^\star)}{m}
+
\operatorname{seg}(s,\mathcal C^\star)\log_2|\mathcal C^\star|.
\]
\begin{equation}\label{eq:corollary}
\begin{aligned}
\mathbb{E}_{s\sim Q_\theta^+}
\bigl[\psi_m^\star(s)\bigr]
&\le
\frac{1}{m}
\sum_{i=1}^{m}
\psi_m^\star(s_i)
\\
&\quad+
U_m^\star
\sqrt{
\frac{
B^\star\ln 2+\ln(m/\delta)
}{
2m
}
}.
\end{aligned}
\end{equation}
\end{proposition}

Equation~\eqref{eq:corollary} makes the roles of $B^\star$ and $K^\star$ explicit: a smaller universal dictionary and fewer reusable segments per successful trajectory directly imply a smaller out-of-sample description length for future successful trajectories. Combined with the EM picture of Section~\ref{sec:idio}---in which the E-step reduces $\mathcal{L}_{\mathrm{BDL}}$ in $\mathcal{C}$ and the M-step reduces the empirical segmentation cost at fixed $\mathcal{C}$---this corollary is what justifies the per-trajectory penalty in Eq.~\eqref{eq:rl-mdl} as a generalization-improving regularizer.

\section{Implicit MDL: SkillRL and ERL Objective Decompositions}\label{app:implicit-mdl}

\paragraph{SkillRL as implicit library compression.}
SkillRL~\citep{xia2026skillrl} constructs a hierarchical skill library $\mathcal{S} = \mathcal{S}_g \cup \mathcal{S}_k$ (general and task-specific skills) via teacher-model distillation. Each skill $s_j\in\mathcal{S}$ is a natural-language behavioral rule compressed from multiple trajectories by an external teacher model. This distillation can be viewed as constructing an implicit codebook $\mathcal{C}_{\mathrm{skill}}$ where each entry encodes a reusable behavioral subroutine, and the skill-evolution and retrieval hyperparameters effectively manage which parts of this library are exposed to the base policy. However, SkillRL does not penalize the segmentation complexity of successful trajectories during RL updates: the policy may still combine retrieved skills in arbitrarily complex ways. At a high level, SkillRL optimizes
\begin{equation}\label{eq:skillrl}
\begin{gathered}
J_{\mathrm{SkillRL}}(\theta,\mathcal{S})
= \mathbb{E}_{\mathcal{B}}\!\left[
\frac{1}{n}\sum_{i=1}^{n} R(\tau_i^{\mathcal{S}})
\right],\\
\mathcal{S}
\leftarrow \texttt{Evolve}\!\left(
\mathcal{S}, \pi_\theta, \mathcal{D}_{\mathrm{val}}
\right).
\end{gathered}
\end{equation}
where $\tau_i^{\mathcal{S}}$ denotes a trajectory generated with retrieved skills. Moreover, their evolving library $\mathcal{S}$ does not compresses the dictionary term $L_{\mathrm{skill\_dict}}$ since they only allow addition of skills, the objective contains no explicit success-conditioned term of the form $\mathds{1}\{R(\tau_i)=1\}\cdot\mathrm{SegCost}(\cdot)$ during policy optimization, and the top-$K$ retrieval only controls memory exposure rather than shared segmentation cost across successful trajectories.

\paragraph{ERL as implicit instance-level internalization.}
ERL~\citep{shi2026experiential} implements an experience--reflection--consolidation loop: for a failed or under-performed trajectory $\tau^{(1)}$, the policy generates a structured reflection $\Delta$ and a corrected trajectory $\tau^{(2)}$, and an internalization loss $\mathcal{L}_{\mathrm{distill}}(\theta) = -\mathbb{E}[\mathds{1}(R(\tau^{(2)})>0)\log\pi_\theta(\tau^{(2)}\mid d)]$ trains the base policy to reproduce successful corrections \emph{without} the reflection context. In MDL terms, ERL starts with a two-part code $L(\Delta)+L(\tau^{(2)}\mid\Delta)$ and compresses it into a single-part code $L(\tau^{(2)})$ by absorbing the codebook (reflections) into model weights---a genuine form of description-length minimization, but at the \emph{instance} level: each reflection addresses a single episode rather than capturing structural patterns shared across trajectories. The effective ERL objective is
\begin{equation}\label{eq:erl}
\begin{aligned}
J_{\mathrm{ERL}}(\theta)
&= \mathbb{E}_{\mathcal{B}}\!\Biggl[
\frac{1}{n}\sum_{i=1}^{n}
\Bigl(
R(\tau_i^{(1)}) \\
&\quad {}-\eta\,
\mathds{1}\!\left\{
\substack{
R(\tau_i^{(1)})=0,
R(\tau_i^{(2)})=1
}
\right\} \\
&\quad {}\cdot
\mathrm{KL}\!\left(
\pi_\theta(\cdot\mid d_i,\Delta_i)
\,\middle\|\,
\pi_\theta(\cdot\mid d_i)
\right)
\Bigr)
\Biggr].
\end{aligned}
\end{equation}
where the indicator activates the reverse-KL term only when the first attempt fails and the reflected second attempt succeeds. ERL therefore compresses an instance-specific two-part code into model weights but does not search for a shared structural dictionary across the successful batch.

\section{Algorithm}\label{app:algorithm}

Algorithm~\ref{alg:reuserl} summarizes the full \methodname{} training loop.

\begin{algorithm}[h]
\caption{\methodname{}: BPE Skill-Compression RL with Segmentation Cost}\label{alg:reuserl}
\DontPrintSemicolon
\KwIn{Policy $\pi_\theta$; skill alphabet $\Sigma$; efficiency coefficient $\lambda$; max steps $T$; phrase cap $L$; global FIFO success buffer $\mathcal{G}$ (optional).}
\For{\textnormal{each training step}}{
  Roll out $\pi_\theta$ on a batch of tasks to obtain trajectories $\{\tau_1,\ldots,\tau_n\}$\;
  Compute task rewards $R(\tau_i)\in\{0,1\}$\;
  $\mathcal{B}^+ \gets \{\tau_i : R(\tau_i)=1\}$ \tcp*{successful trajectories}
  \If{$\mathcal{B}^+ = \emptyset$}{
    Update $\pi_\theta$ via GRPO with original rewards $\{R(\tau_i)\}$\;
    \textbf{continue}\;
  }
  \For{$\tau_i \in \mathcal{B}^+$}{
    $s_i \gets \varphi(\tau_i)$ \tcp*{rule-based projection over action verbs}
  }
  Append $\{s_i : \tau_i\in\mathcal{B}^+\}$ to global FIFO buffer $\mathcal{G}$\ (optional);
  $\mathcal{C} \gets \texttt{BPE-MDL}(\{s_i : \tau_i\in\mathcal{B}^+\cup\mathcal{G}\};\,L)$ \tcp*{greedy dictionary minimizing $\mathcal{L}_{\mathrm{BDL}}$}
  $\ell_i  \gets \mathrm{seg}(s_i, \mathcal{C})/T$ for each $\tau_i \in \mathcal{B}^+$\;
  $r_i \gets R(\tau_i) - \lambda\cdot\mathds{1}\{R(\tau_i)=1\}\cdot\ell_i$ for $\tau_i \in \mathcal{B}^+$\;
  $r_i \gets R(\tau_i)$ for $\tau_i \notin \mathcal{B}^+$\;
  Update $\pi_\theta$ via GRPO with shaped rewards $\{r_i\}$\;
}
\end{algorithm}

\section{More Training and Evaluation Details}\label{appendix:experiment_details}
\paragraph{Task Details}ALFWorld requires an agent to complete household tasks (e.g., ``heat a potato and place it on the counter'') by issuing natural-language actions in a partially observable simulated environment; the maximum interaction horizon is $T=50$ steps and a successful trajectory receives terminal reward $+10$, which means $R(\tau)\in\{0,10\}$ in practice because following the verl-agent \citep{feng2025gigpo} framework we have a invalid action penalty set as 0.01 and this sometimes prevents the model from collapsing their outputs. This mechanism has applied to all environments used here. TextWorld-Cooking is the cooking game under the TextWorld framework: the agent must read the recipe, find all ingredients, process them (cut/slice/dice), cook them, prepare the meal, and eat it; ordering matters and some processing steps could directly lead to failures such as burning the food, an irreversible terminal failure. The current generated training and evaluation data both follow a $4{:}1$ simple-hard division/splits based on the rooms, ingredients, is\_cook, and is\_cut settings; the hard setting is room=6, ingredients=2, is\_cook=True, is\_cut=True and the simple setting refers to room=1, ingredients=1, is\_cook=True|False, is\_cut=True|False. We use $T=40$ and terminal reward $+10$ on success. Countdown-Stepwise is a stepwise variant of the Countdown game from Reasoning-Gym in which each action is either a single arithmetic operation between two numbers from the current number pool (\textsc{Add}, \textsc{Sub}, \textsc{Mul}, or \textsc{Div}), \textsc{Rollback} to undo the previous operation, or \textsc{Reset} to restore the initial pool. We train and evaluate on $3$- to $4$-number scenarios with target values in $[10,999]$, $T=30$, and terminal reward $+10$ on success. All three environments verify intermediate steps so that successful trajectories are process-correct.

\paragraph{Training Details} Training is performed on Nvidia RTX Pro 6000 96\,GB GPUs with tensor parallelism, using the verl-agent framework~\citep{feng2025gigpo} on top of verl~\citep{sheng2025hybridflow}. The efficiency coefficient is $\lambda=10$ for both penalty-based variants as mentioned previously our binary reward setting is $R(\tau)\in\{0,10\}$, following \citet{feng2025gigpo}. The evaluation inference engine is based on vLLM~\citep{vllm}. The skill-projection rules and BPE search are CPU-only and run inline with the GRPO update loop; the global FIFO success buffer is sized at $256$ skill sequences.

\paragraph{Evaluation Details} The Qwen2.5-1.5B-Instruct is using temperature=0.4 for sampling while the Qwen3-1.7B is following the recommended sampling setting: the temperature=0.7 top\_p=0.8, and top\_k=20, respectively. The Alfworld's two test splits: IID and OOD have 140 and 134 scenes, respectively. The Textworld-Cooking and Countdown-Stepwise have 1000 and 1024 generated test examples.

\section{Atomic Skill Alphabets and Rule-Based Mappings}
\label{app:skills}

This appendix details the deterministic verb-matching rules that implement
the skill projection $\varphi$ for each environment.
Table~\ref{tab:skill-mappings} lists the action prefix patterns and their
atomic-skill targets. We additionally provide one worked example per
environment to illustrate how a raw trajectory becomes an atomic-skill sequence.

\begin{table*}[t]
\centering
\small
\setlength{\tabcolsep}{5pt}
\renewcommand{\arraystretch}{1.08}
\begin{tabularx}{\textwidth}{
@{}
>{\raggedright\arraybackslash}X
>{\raggedright\arraybackslash}p{0.32\textwidth}
@{}}
\toprule
\textbf{Action prefix / form} & \textbf{Atomic skill} \\
\midrule

\multicolumn{2}{@{}l}{\textbf{ALFWorld} ($|\Sigma|=5$)} \\
\cmidrule(lr){1-2}
\texttt{go to *}, \texttt{open *} (not carrying)
& \skill{Explore} \\
\texttt{go to *}, \texttt{open *} (carrying)
& \skill{Transport} \\
\texttt{look}, \texttt{examine *}
& \skill{Explore} \\
\texttt{take *}
& \skill{Take} \\
\texttt{move *}, \texttt{put *}
& \skill{Deliver} \\
\texttt{heat *}, \texttt{cool *}, \texttt{clean *},
\texttt{use *}, \texttt{light *}
& \skill{Transform} \\

\addlinespace[0.6em]
\multicolumn{2}{@{}l}{\textbf{TextWorld-Cooking} ($|\Sigma|=10$)} \\
\cmidrule(lr){1-2}
\texttt{examine cookbook}
& \skill{Read\_Recipe} \\
\texttt{look}, \texttt{inventory}, \texttt{eat *},
\texttt{examine *} (other)
& \skill{Inspect} \\
\texttt{go *}
& \skill{Explore} \\
\texttt{open *}, \texttt{close *}
& \skill{Open} \\
\texttt{take *}
& \skill{Take} \\
\texttt{drop *}, \texttt{put *}, \texttt{insert *}
& \skill{Deliver} \\
\texttt{chop *}, \texttt{slice *}, \texttt{dice *}
& \skill{Cut} \\
\texttt{cook *}
& \skill{Cook} \\
\texttt{prepare meal}
& \skill{Prepare\_Meal} \\
\texttt{eat meal}
& \skill{Eat\_Meal} \\

\addlinespace[0.6em]
\multicolumn{2}{@{}l}{\textbf{Countdown-Stepwise} ($|\Sigma|=26$)} \\
\cmidrule(lr){1-2}
\(\texttt{op}(+,n_1,n_2)\)
& \skill{OP\_Add-}\(\langle r_1\rangle\)\skill{-}\(\langle r_2\rangle\) \\
\(\texttt{op}(-,n_1,n_2)\)
& \skill{OP\_Sub-}\(\langle r_1\rangle\)\skill{-}\(\langle r_2\rangle\) \\
\(\texttt{op}(\ast,n_1,n_2)\)
& \skill{OP\_Mul-}\(\langle r_1\rangle\)\skill{-}\(\langle r_2\rangle\) \\
\(\texttt{op}(/,n_1,n_2)\)
& \skill{OP\_Div-}\(\langle r_1\rangle\)\skill{-}\(\langle r_2\rangle\) \\
\texttt{rollback}
& \skill{Rollback} \\
\texttt{reset}
& \skill{Reset} \\
\bottomrule
\end{tabularx}
\caption{
Atomic-skill mappings for all three environments.
In ALFWorld, the carry flag is updated when an effective
\texttt{take}, \texttt{move}, or \texttt{put} action changes the inventory.
In TextWorld-Cooking, \texttt{open} and \texttt{close} are merged into
\skill{Open}. In Countdown-Stepwise, \(r_1,r_2\) are target-relative
role tags sorted alphabetically before joining.
}
\label{tab:skill-mappings}
\end{table*}

\subsection{ALFWorld ($|\Sigma|=5$)}

\paragraph{Worked example.}
{\small\raggedright
\textbf{Goal.}
\emph{``put a clean ladle on countertop 1''}

\par\smallskip
\textbf{Trajectory.}\par
\texttt{go to drawer 1}\\
\texttt{open drawer 1};\\
\texttt{take ladle 1 from drawer 1};\\
\texttt{go to sinkbasin 1};\\
\texttt{clean ladle 1 with sinkbasin 1};\\
\texttt{go to countertop 1};\\
\texttt{put ladle 1 on countertop 1}.

\par\smallskip
\textbf{Projection.}
\[
\begin{aligned}
\varphi(\tau)=(&
\skill{Explore}, \skill{Explore}, \skill{Take},\\
&\skill{Transport}, \skill{Transform}, \skill{Transport},\\
&\skill{Deliver}).
\end{aligned}
\]
\par}

\subsection{TextWorld-Cooking ($|\Sigma|=10$)}

\paragraph{Worked example.}
{\small\raggedright
\textbf{Trajectory.}\par
\texttt{examine cookbook}\\
\texttt{open fridge};\\
\texttt{take carrot from fridge}\\
\texttt{dice carrot};\\
\texttt{cook carrot with stove};\\
\texttt{prepare meal}\\
\texttt{eat meal}.

\par\smallskip
\textbf{Projection.}
\[
\begin{aligned}
\varphi(\tau)=(&
\skill{Read\_Recipe}, \skill{Open}, \skill{Take},\\
&\skill{Cut}, \skill{Cook}, \skill{Prepare\_Meal},\\
&\skill{Eat\_Meal}).
\end{aligned}
\]
\par}

\subsection{Countdown-Stepwise ($|\Sigma|=26$)}

Each step's action is parsed into one of three forms:
\(\mathrm{Op}(\langle\mathrm{op}\rangle,n_1,n_2)\),
\(\mathrm{Rollback}\), or \(\mathrm{Reset}\).
For an \(\mathrm{Op}\) action with operands \((n_1,n_2)\) and target
\(T_{\mathrm{tgt}}\), each operand is assigned a target-relative role:
{\small
\[
\mathrm{role}(n)=
\begin{cases}
\mathrm{near\_target},
& \text{if } |n-T_{\mathrm{tgt}}|
  \leq 0.10|T_{\mathrm{tgt}}|,\\
\mathrm{small},
& \text{if } n<T_{\mathrm{tgt}}-0.10|T_{\mathrm{tgt}}|,\\
\mathrm{large},
& \text{otherwise.}
\end{cases}
\]
}
The two role tags are sorted alphabetically before being joined, treating
the operation as unordered. Thus, \texttt{op(+, 1, 10)} and
\texttt{op(+, 10, 1)} induce the same unordered role pair. With four
operators and six unordered role pairs from
\(\{\mathrm{large},\mathrm{near\_target},\mathrm{small}\}\), this yields
\(4\times 6=24\) operation skills, plus \skill{Rollback} and \skill{Reset},
for a total of \(26\).

\paragraph{Worked example.}
{\small\raggedright
\textbf{Puzzle.}
\(\{80,2,28,1\}\) with target \(T_{\mathrm{tgt}}=54\).
Here \(n>59.4\) is \(\mathrm{large}\),
\(|n-54|\leq 5.4\) is \(\mathrm{near\_target}\),
and all remaining values are \(\mathrm{small}\).

\par\smallskip
\textbf{Trajectory with roles.}\par
\texttt{op(-, 80, 28)}:
\(80\mapsto\mathrm{large}\), \(28\mapsto\mathrm{small}\);\\
\texttt{op(+, 52, 2)}:
\(52\mapsto\mathrm{near\_target}\), \(2\mapsto\mathrm{small}\);\\
\texttt{op(*, 54, 1)}:
\(54\mapsto\mathrm{near\_target}\), \(1\mapsto\mathrm{small}\).

\par\smallskip
\textbf{Projection.}
\[
\begin{aligned}
\varphi(\tau)=(&
\skill{OP\_Sub-large-small},\\
&\skill{OP\_Add-near\_target-small},\\
&\skill{OP\_Mul-near\_target-small}).
\end{aligned}
\]
\par}


\section{Ablation Details}\label{app:ablations}

\subsection{More Details about BPE Merging}
\label{app:bpe-merging}

\paragraph{Setup.}
We quantify the approximation gap between greedy BPE-merging and exhaustive brute-force (BF) search over all candidate dictionaries. The benchmark is a synthetic corpus of around 200 random skill sequences with dummy atomic skills ($|\Sigma|=5$, i.e.\ $\{A, B, C, D, E\}$), with sequence lengths sampled uniformly from 2 to 6 and each group consisting of 10 sequences. We compare against two admissible brute-force variants. \emph{BF-all-singletons} requires every corpus singleton to be retained in the candidate dictionaries, matching BPE's structural invariant. \emph{BF-original} drops this constraint and attains the true $\min L_{BDL}$ by pruning singletons whose every occurrence is covered by a longer phrase.

\paragraph{Approximation quality.}
BPE attains a mean penalty of $7.25$, within $0.14\%$ of BF-all-singletons ($7.24$) and recovering $99.02\%$ of its non-singleton phrases, showing that the greedy schedule is essentially exact within its own search space. BF-original reaches $7.01$, a $3.3\%$ reduction below BPE's upper bound, but shares only $66.34\%$ of non-singleton phrases with BPE and requires exhaustive search.

\paragraph{Runtime.}
BPE takes $0.053$~ms per group ($10.9$~ms total), versus $11.86$~ms ($2430.6$~ms total) for BF-all-singletons and $118.58$~ms ($24308.5$~ms total) for BF-original, giving $285\times$ and $2231\times$ speedups respectively. The combination of a tight approximation gap and negligible runtime is what makes inline E-step dictionary extraction practical within the GRPO loop.

\subsection{More Details about the Global Success Buffer Ablation}\label{app:ablation_global_buffer}
\paragraph{Countdown-Stepwise closing operations (Table~\ref{tab:cd-ops}).}
Countdown-Stepwise has exactly two terminal-loss conditions: \texttt{Timeout} (the $30$-step budget is exhausted) and \texttt{Stuck} (the model issues \textsc{Op} or an unparsable action while the working pool has fewer than two numbers, which the env treats as unrecoverable). Within-episode \textsc{Reset}/\textsc{Rollback} counts and invalid-action counts are descriptors of rollout noise but not termination criteria. Table~\ref{tab:cd-ops} reports the raw count of \emph{closing} operations per method---the operation that puts the pool at $[\,\text{target}\,]$---broken down by operator. Only the buffered segmentation cost lifts the rare $*$ and $/$ closures (e.g.\ $\mathrm{OP}(*,1,T)$ to absorb a leftover $1$); removing the buffer collapses the \textsc{Mul} count from $36$ to $18$, essentially matching GRPO ($20$).

\subsection{Random Atomic Skill Projection}\label{app:skill_projection}

To test whether our gains stem from compression over a meaningful abstraction rather than from the compression machinery itself, we replace the rule-based projection $\varphi$ on Countdown-Stepwise, which is the environment with the largest and most structured alphabet ($|\Sigma|=26$), with a fixed random assignment of each parsed action to a uniformly drawn skill, preserving determinism and alphabet size while destroying the target-relative role semantics; all other components are unchanged. The random projection attains $75.23 \pm 0.15$ Pass@1 over three seeds: well above vanilla GRPO ($68.46$), since in a budget-exhaustion environment any consistent compression pressure discourages wasted steps, yet below even the degenerate pure round-length baseline ($77.02$) and far below \methodname{}-SegCost ($80.37$). Under an incoherent alphabet, successful trajectories share little compressible substructure, so BPE merges spurious co-occurrences and the segmentation cost collapses into a noisy variant of the raw-length penalty. The observed performance gains therefore require both ingredients: the MDL objective and an abstraction that preserves behaviorally meaningful distinctions.

\subsection{\methodname{} with Larger Models}\label{app:larger_models}

To test whether the benefits of structural compression persist beyond the 1.5--1.7B scale of our main experiments, we repeat the Countdown-Stepwise comparison with Qwen2.5-3B-Instruct and Qwen2.5-7B-Instruct, keeping all training hyperparameters fixed and using a global success buffer of size 1024; each configuration is evaluated over three seeds. As Table~\ref{tab:model_scale} shows, the strict ordering \methodname{}-SegCost $\succ$ pure round-length $\succ$ vanilla GRPO from Table~\ref{tab:main} is preserved at both scales. Two trends are worth noting. First, stronger base models close much of the gap on their own: vanilla GRPO rises from $68.46$ at 1.5B to $79.98$ at 7B; yet the segmentation cost still adds a consistent margin over both baselines ($+3.71/+1.20$ at 3B and $+7.10/+0.98$ at 7B over GRPO and round-length, respectively). Second, the advantage of \methodname{} over uniform length pressure narrows with scale but does not vanish, suggesting that larger models internalize some reusable solver structure from success rewards alone, while the explicit MDL signal remains complementary rather than redundant. We therefore find no evidence that our conclusion on the MDL compression over skills is only an artifact of small model capacity.

\begin{table}[t]
\centering
\small
\begin{tabular}{lcc}
\toprule
Method & 3B & 7B \\
\midrule
Vanilla GRPO & $74.28 \pm 0.41$ & $79.98 \pm 1.28$ \\
Pure Round-Length & $76.79 \pm 0.50$ & $86.10 \pm 0.65$ \\
\methodname{}-SegCost & $\mathbf{77.99 \pm 0.20}$ & $\mathbf{87.08 \pm 0.46}$ \\
\bottomrule
\end{tabular}
\caption{Countdown-Stepwise Pass@1 (mean $\pm$ std over three runs at a fixed seed) with larger Qwen2.5-Instruct backbones. The ordering \methodname{}-SegCost $\succ$ pure round-length $\succ$ vanilla GRPO from Table~\ref{tab:main} holds at both scales.}
\label{tab:model_scale}
\end{table}

\begin{table}[!t]
\centering
\footnotesize
\setlength{\tabcolsep}{4.0pt}
\renewcommand{\arraystretch}{1.05}
\begin{tabular}{@{}lcccc@{}}
\toprule
\textbf{Closing OP}
& \makecell[c]{\textbf{GRPO}}
& \makecell[c]{\textbf{Pure}\\\textbf{Length}}
& \makecell[c]{\textbf{SegCost}\\\textbf{(buf)}}
& \makecell[c]{\textbf{SegCost}\\\textbf{(no buf)}} \\
\midrule
$+$ (\textsc{Add}) & 359 & 393 & 392 & 354 \\
$-$ (\textsc{Sub}) & 319 & 368 & 392 & 335 \\
$*$ (\textsc{Mul}) & 20 & 39 & \textbf{36} & 18 \\
$/$ (\textsc{Div}) & 2 & 0 & 1 & 0 \\
\bottomrule
\end{tabular}
\caption{Closing-operation counts on successful Countdown-Stepwise trajectories. The buffered segmentation cost variant retains the multiplicative-closure motifs that the no-buffer variant prunes.}
\label{tab:cd-ops}
\end{table}

\paragraph{TextWorld-Cooking mechanism (Table~\ref{tab:tw-mechanism}).}
We call a TW-Cooking recipe \texttt{cook=False + cookbook-says-cook} when its \texttt{cook} gate is off---the ingredient arrives in the observation already in cooked form, e.g.\ \texttt{roasted red apple} on the counter---but the cookbook page still mentions the cooking of the ingredients; attempting to cook the ingredient a second time burns it. There are $260$ such games in the $1000$-game validation split. Table~\ref{tab:tw-mechanism} reports, per method, (1) the share of decision steps whose chain-of-thought references a cooking verb (\texttt{cook|roast|fry|grill|stove|oven}), (2) the conditional rate at which such steps emit a \texttt{cook} action, and (3) the per-game emission rate of \texttt{cook} actions split by cookbook content. The four policies think about cooking at comparable rates (24--61\%); only the cookbook-literal fixed-code policy commits to a \texttt{cook} action conditional on that thought ($73\%$), while both segmentation cost variants suppress the action to $0.3\%$.

\begin{table}[!t]
\centering
\small
\setlength{\tabcolsep}{5pt}
\renewcommand{\arraystretch}{1.08}
\resizebox{\columnwidth}{!}{%
\begin{tabular}{@{}lcccc@{}}
\toprule
\textbf{Method}
& \makecell[c]{\textbf{Thought}\\\textbf{mentions}\\\texttt{cook}}
& \makecell[c]{\textbf{Action}\\\texttt{cook}\\\textbf{$\mid$ thought}}
& \makecell[c]{\textbf{Cookbook}\\\textbf{says cook}}
& \makecell[c]{\textbf{Cookbook}\\\textbf{silent}} \\
\midrule
GRPO              & 24 & 8          & 12          & 16 \\
Pure-Length       & 42 & \textbf{73} & \textbf{94} & 0  \\
SegCost (buf)     & 34 & 0.3        & 1           & 0  \\
SegCost (no buf)  & \textbf{61} & 0.3 & 1           & 0  \\
\bottomrule
\end{tabular}%
}
\caption{Per-method thought-vs-action breakdown on
\texttt{cook=False + cookbook-says-cook} TW-Cooking games. The first
two columns report step-level rates: how often thoughts mention
\texttt{cook}, and how often the resulting action is \texttt{cook}
conditional on such a mention. The last two columns report game-level
cook-action emission rates by cookbook content.}
\label{tab:tw-mechanism}
\end{table}

\section{Extended Case-Study Material and Learned-Dictionary Examples}
\label{app:case-study}

This appendix collects the per-method trajectory analyses supporting Section~\ref{sec:case-study}, together with representative high-frequency non-singleton phrases from the learned dictionary $\widehat{\mathcal{C}}$ and their correspondence to skills distilled by SkillRL~\citep{xia2026skillrl}. All trajectory-level numbers below are computed on the same single seed's step-$300$ eval JSON per method.

\paragraph{ALFWorld (Table~\ref{tab:case-alfworld-detail}).}
Failed ALFWorld episodes all hit the $50$-step cap; behavioral signal therefore lives in the action-quality metrics on \emph{successful} episodes. Vanilla GRPO leaks roughly one-quarter of its slots to ``Nothing happens.'' returns and repeats earlier actions on $40\%$ of slots; both \methodname{} variants and pure-round-length reduce these pathologies to single digits.

\begin{table}[!t]
\centering
\footnotesize
\setlength{\tabcolsep}{3.0pt}
\renewcommand{\arraystretch}{1.05}
\resizebox{\columnwidth}{!}{%
\begin{tabular}{@{}lcccc@{}}
\toprule
\textbf{ALFWorld IID (successes)}
& \makecell[c]{\textbf{GRPO}}
& \makecell[c]{\textbf{Pure}\\\textbf{Length}}
& \makecell[c]{\textbf{SegCost}\\\textbf{(buf)}}
& \makecell[c]{\textbf{SegCost}\\\textbf{(no buf)}} \\
\midrule
Avg succ.-episode length (steps)
& 15.3 & \textbf{8.7} & 9.3 & 8.4 \\
``Nothing-happens'' rate (\%)
& 23.5 & \textbf{1.4} & 1.8 & 3.9 \\
Action-repetition rate (\%)
& 40.4 & \textbf{9.4} & 14.4 & 9.9 \\
Redundant \texttt{go to} revisits
& 3.4 & 1.4 & 1.9 & \textbf{0.85} \\
\bottomrule
\end{tabular}%
}
\caption{ALFWorld per-method behavioral quality on successful IID episodes. ``Nothing happens.'' is the env response when an action is not admissible from the current state.}
\label{tab:case-alfworld-detail}
\end{table}

\paragraph{ALFWorld failure modes on the OOD split (Table~\ref{tab:alf-failure-modes}).}
Every failed ALFWorld episode hits the $50$-step cap; we label timeouts by what was still missing.
\texttt{never\_picked\_target} means no \emph{effective} \texttt{take} of the goal object; \texttt{two\_partial\_deliver} means only one of two instances was placed on \emph{pick\_two\_obj\_and\_place}; \texttt{picked\_no\_transform} means the agent took the goal object but never executed an effective \texttt{clean}/\texttt{heat}/\texttt{cool} (checked only after a successful take).
On the OOD split ($134$ logical games, $938$ episodes), \methodname{}-SegCost improves SC@7 by $1.49$\,pp over pure-round-length ($93.28$ vs.\ $91.79$) with $21$ fewer failed rollouts ($50$ vs.\ $71$).
The gap is driven almost entirely by \texttt{never\_picked\_target} on held-out \emph{multi-phase} tasks---especially \emph{pick\_clean\_then\_place\_in\_recep} ($n{=}217$), where Pure-Length accumulates $19$ search failures vs.\ $6$ for segmentation cost---not by \texttt{two\_partial\_deliver} (three failures each on Two-Obj) or lamp-ordering errors (both methods reach $100\%$ on Look+Light).
Overall OOD \texttt{never\_picked\_target} falls from $64$ to $44$, matching the analysis that uniform step counting over-penalizes exploratory prefixes that the BPE dictionary treats as reusable structure.
\begin{table}[!t]
\centering
\footnotesize
\setlength{\tabcolsep}{3.0pt}
\renewcommand{\arraystretch}{1.05}
\resizebox{\columnwidth}{!}{%
\begin{tabular}{@{}lcccc@{}}
\toprule
\textbf{Failed OOD episodes (raw counts)}
& \makecell[c]{\textbf{GRPO}}
& \makecell[c]{\textbf{Pure}\\\textbf{Length}}
& \makecell[c]{\textbf{SegCost}\\\textbf{(buf)}}
& \makecell[c]{\textbf{SegCost}\\\textbf{(no buf)}} \\
\midrule
\multicolumn{5}{@{}l}{\emph{Clean+Place ($n{=}217$)}} \\
\,\,\, total fails
& 45 & 19 & \textbf{6} & 12 \\
\,\,\, \texttt{never\_picked\_target}
& 42 & 19 & \textbf{6} & 12 \\
\,\,\, \texttt{picked\_no\_transform}
& 3 & 0 & 0 & 0 \\
\midrule
\multicolumn{5}{@{}l}{\emph{Two-Obj+Place ($n{=}119$)}} \\
\,\,\, total fails
& 61 & 9 & \textbf{5} & 10 \\
\,\,\, \texttt{never\_picked\_target}
& 39 & 5 & \textbf{1} & 5 \\
\,\,\, \texttt{two\_partial\_deliver}
& 20 & 3 & 3 & 5 \\
\midrule
\multicolumn{5}{@{}l}{\emph{All OOD tasks}} \\
\,\,\, total fails
& 231 & 71 & \textbf{50} & 88 \\
\,\,\, \texttt{never\_picked\_target}
& 181 & 64 & \textbf{44} & 67 \\
\bottomrule
\end{tabular}%
}
\caption{ALFWorld failure-mode counts on the OOD split. Different error cases are shown with the subtasks.}
\label{tab:alf-failure-modes}
\end{table}

\paragraph{TextWorld-Cooking (Table~\ref{tab:case-tw-detail}).}
Failed TW-Cooking episodes fall into three terminal conditions: \texttt{Burned} (the terminal observation contains \texttt{burned}), \texttt{Wrong processing/order} (a non-burn rule violation such as trying to slice an ingredient that should be chopped), and \texttt{Timeout} (the $40$-step budget is exhausted without either lost marker). \texttt{Pure-Length} fails \emph{fast}---mean failed-episode length $10.6$ steps---and almost always by burning; both segmentation cost variants fail near the step ceiling like GRPO.

\begin{table}[!t]
\centering
\footnotesize
\setlength{\tabcolsep}{3.0pt}
\renewcommand{\arraystretch}{1.05}
\resizebox{\columnwidth}{!}{%
\begin{tabular}{@{}lcccc@{}}
\toprule
\textbf{TW-Cooking}
& \makecell[c]{\textbf{GRPO}}
& \makecell[c]{\textbf{Pure}\\\textbf{Length}}
& \makecell[c]{\textbf{SegCost}\\\textbf{(buf)}}
& \makecell[c]{\textbf{SegCost}\\\textbf{(no buf)}} \\
\midrule
Succ.-episode length (steps, mean)
& 12.06 & \textbf{8.23} & 9.00 & 8.83 \\
Failed-episode length (steps, mean)
& 30.5 & \textbf{10.6} & 29.3 & 29.4 \\
\midrule
\multicolumn{5}{@{}l}{\emph{Failure-mode breakdown (\% of failures)}} \\
\texttt{Burned}
& 21.1 & \textbf{74.0} & 3.3 & 9.7 \\
\texttt{Wrong processing/order}
& 27.1 & 12.0 & \textbf{44.0} & 37.6 \\
\texttt{Timeout}
& 51.8 & 14.0 & 52.7 & 52.7 \\
\bottomrule
\end{tabular}%
}
\caption{TextWorld-Cooking per-method trajectory and failure-mode breakdown.}
\label{tab:case-tw-detail}
\end{table}

\paragraph{Countdown-Stepwise (Table~\ref{tab:case-cd-detail}).}
An episode in Countdown-Stepwise terminates under exactly two negative conditions: (i) \textit{Timeout}, where the agent exhausts the maximum allocation of 30 steps without reaching the target value, and (ii) \textit{Stuck}, a premature failure where the agent attempts to execute an arithmetic operation or emits an unparsable command when fewer than two numbers remain available in the working pool, rendering further evaluation mathematically impossible. Both compression methods, pure-round-length and \methodname{}-SegCost drive the failures toward \texttt{Timeout} (it has stopped getting stuck), while the no-buffer variant exhibits the opposite shift---a side-effect of having deleted \textsc{Reset} from its repertoire (zero RESETs in $99.9\%$ of successes), leaving it unable to recover from early suboptimal operations once the numbers are collapsed.

\begin{table}[!t]
\centering
\footnotesize
\setlength{\tabcolsep}{3.0pt}
\renewcommand{\arraystretch}{1.05}
\resizebox{\columnwidth}{!}{%
\begin{tabular}{@{}lcccc@{}}
\toprule
\textbf{Countdown-Stepwise}
& \makecell[c]{\textbf{GRPO}}
& \makecell[c]{\textbf{Pure}\\\textbf{Length}}
& \makecell[c]{\textbf{SegCost}\\\textbf{(buf)}}
& \makecell[c]{\textbf{SegCost}\\\textbf{(no buf)}} \\
\midrule
Succ.-episode length (steps, mean)
& 6.70 & 4.30 & 4.51 & \textbf{3.98} \\
Avg \textsc{Reset}s per success
& 1.07 & 0.52 & 0.58 & \textbf{0.00} \\
Zero-\textsc{Reset} successes (\%)
& 58.0 & 82.0 & 80.3 & \textbf{99.9} \\
Invalid-action rate, success (\%)
& 5.0 & \textbf{0.3} & 0.5 & 1.0 \\
Action-repetition rate, success (\%)
& 10.7 & 5.5 & \textbf{4.4} & 5.9 \\
\midrule
\multicolumn{5}{@{}l}{\emph{Operation diversity in successful episodes}} \\
\textsc{Mul}/\textsc{Div} share of OPs (\%)
& 3.1 & 6.8 & \textbf{10.9} & 5.5 \\
Avg unique op types per success
& 1.70 & 1.82 & \textbf{1.85} & 1.75 \\
\midrule
\multicolumn{5}{@{}l}{\emph{Env-level failure-mode partition (\% of failures)}} \\
\texttt{Timeout} ($\mathrm{len}{=}30$)
& 68.5 & 82.6 & 87.2 & \textbf{56.8} \\
\texttt{Stuck} ($\mathrm{len}{<}30$, \textsc{Op}/\texttt{INVALID\_PARSE})
& 31.5 & 17.4 & \textbf{12.8} & 43.2 \\
\bottomrule
\end{tabular}%
}
\caption{Countdown-Stepwise per-method trajectory analysis. The only env-level failure modes are \texttt{Timeout} and \texttt{Stuck}.}
\label{tab:case-cd-detail}
\end{table}

\paragraph{Learned dictionary phrases (Table~\ref{tab:top-phrases}) and SkillRL correspondence (Table~\ref{tab:skill-mapping}).}
The phrases recovered by \methodname{}-SegCost reflect environment-specific reusable structure: navigation-and-place skeletons in ALFWorld, recipe-faithful processing chains in TextWorld-Cooking, and multi-step solver templates in Countdown-Stepwise. On ALFWorld the learned non-singleton phrases align closely with skills distilled in SkillRL's Skillbank~\citep{xia2026skillrl}.

\begin{table*}[!t]
\centering
\small
\setlength{\tabcolsep}{6pt}
\renewcommand{\arraystretch}{1.12}
\begin{tabularx}{\textwidth}{@{}
>{\raggedright\arraybackslash}p{0.20\textwidth}
>{\raggedright\arraybackslash}X
@{}}
\toprule
\textbf{Environment}
& \textbf{Representative top non-singleton phrases} \\
\midrule

\textbf{ALFWorld}
&
\begin{tabular}[t]{@{}l@{}}
\skill{Explore} \(\to\) \skill{Take} \\
\skill{Take} \(\to\) \skill{Transport} \(\to\) \skill{Deliver} \\
\skill{Explore} \(\to\) \skill{Take} \(\to\) \skill{Transport}
  \(\to\) \skill{Deliver} \\
\skill{Take} \(\to\) \skill{Transport} \(\to\) \skill{Transform}
  \(\to\) \skill{Deliver}
\end{tabular}
\\

\midrule
\textbf{TextWorld-Cooking}
&
\begin{tabular}[t]{@{}l@{}}
\skill{Read\_Recipe} \(\to\) \skill{Take} \\
\skill{Take} \(\to\) \skill{Cut} \\
\skill{Cut} \(\to\) \skill{Cook} \\
\skill{Cook} \(\to\) \skill{Prepare\_Meal} \(\to\) \skill{Eat\_Meal}
\end{tabular}
\\

\midrule
\textbf{Countdown-Stepwise}
&
\begin{tabular}[t]{@{}l@{}}
\skill{OP\_Sub-large-near\_target}
  \(\to\) \skill{OP\_Div-near\_target-small} \\
\skill{OP\_Add-small-small}
  \(\to\) \skill{OP\_Sub-large-near\_target} \\
\skill{OP\_Sub-large-large}
  \(\to\) \skill{OP\_Add-near\_target-small}
\end{tabular}
\\
\bottomrule
\end{tabularx}
\caption{
Representative high-frequency non-singleton phrases of the extracted
dictionary \(\widehat{\mathcal{C}}\) on each environment.
}
\label{tab:top-phrases}
\end{table*}

\begin{table*}[!t]
\centering
\footnotesize
\renewcommand{\arraystretch}{1.1}
\begin{tabular}{@{}>{\raggedright\arraybackslash}p{\dimexpr0.5\textwidth-\tabcolsep\relax}>{\raggedright\arraybackslash}p{\dimexpr0.5\textwidth-\tabcolsep\relax}@{}}
\toprule
\textbf{Extracted phrase (ALFWorld)} & \textbf{Closest SkillRL Skillbank skill} \\
\midrule
\skill{Transport} $\to$ \skill{Transform}
& Approach a tool/appliance with the carried item and apply it (e.g.\ clean, heat, cool) \\
\skill{Explore} $\to$ \skill{Explore} $\to$ \skill{Explore} $\to$ \skill{Explore}
& Systematic-search routine over candidate receptacles when the target location is unknown \\
\skill{Transport} $\to$ \skill{Deliver}
& Carry-and-place: move to the destination receptacle and place the held object \\
\skill{Explore} $\to$ \skill{Take} $\to$ \skill{Transport} $\to$ \skill{Deliver}
& Canonical pick-and-place skeleton for simple \emph{Place} tasks \\
\skill{Take} $\to$ \skill{Transport} $\to$ \skill{Transform} $\to$ \skill{Deliver}
& Clean/heat/cool-then-place: process the object at the appliance before delivery \\
\bottomrule
\end{tabular}
\caption{Mapping between high-frequency non-singleton phrases recovered by \methodname{}-SegCost on ALFWorld and the closest skill description in SkillRL's Skillbank~\citep{xia2026skillrl}. The correspondence is qualitative; the rule-based projection only carries verb-class membership, so a single learned phrase may match more than one Skillbank skill.}
\label{tab:skill-mapping}
\end{table*}

\end{document}